\documentclass{article}

\usepackage[preprint]{neurips_2026}

\usepackage[utf8]{inputenc}
\usepackage[T1]{fontenc}

\usepackage{microtype}
\usepackage{graphicx}
\usepackage{subcaption}
\usepackage{booktabs}
\usepackage{makecell}
\usepackage{longtable}
\usepackage{array}
\usepackage{multirow}
\usepackage{amsmath,amssymb,amsfonts}
\usepackage{amsthm}
\usepackage{bm}
\usepackage{mathtools}
\usepackage{algorithm}
\usepackage{algorithmic}
\usepackage{enumitem}
\usepackage{nicefrac}
\usepackage{xcolor}
\usepackage{url}
\usepackage{hyperref}
\usepackage{float}

\DeclareMathOperator{\Var}{Var}
\DeclareMathOperator{\Diag}{Diag}
\DeclareMathOperator{\sech}{sech}
\DeclareMathOperator{\softmax}{softmax}
\DeclareMathOperator{\blkdiag}{blkdiag}
\newcommand{\DeltaProd}{\mathcal X}
\newcommand{\Tspace}{\mathcal T}
\newcommand{\Wt}{W_{\mathrm{tan}}}

\newtheorem{theorem}{Theorem}
\newtheorem{lemma}[theorem]{Lemma}
\newtheorem{proposition}[theorem]{Proposition}
\newtheorem{corollary}[theorem]{Corollary}
\theoremstyle{definition}

\theoremstyle{remark}

\title{Sharp Spectral Thresholds for Logit Fixed Points}

\author{%
  Tongxi Wang\\
  Southeast University\\
  Nanjing, China\\
  \texttt{tongxi\_wang@seu.edu.cn}
}

\begin{document}

\maketitle

\begin{abstract}
Softmax feedback systems are a common mathematical core of entropy-regularized reinforcement learning, logit game dynamics, population choice, and mean-field variational updates.  Their central stability question is simple: when does a self-reinforcing softmax system produce a unique and globally predictable outcome?  Classical theory gives a conservative answer.  By treating softmax as a unit-scale response, it certifies stability only in a strongly randomized regime.

We prove that the classical approach misses an entire stable regime and does not identify the point at which the qualitative change truly occurs.  For finite-dimensional affine logit systems, the sharp dimension-free Euclidean threshold is
\[
    \beta\|\Pi W\Pi\|_{\Tspace\to\Tspace}<2,
\]
rather than the previously used condition, which certifies stability only while the softmax system remains safely over-regularized.  Our theorem fills the previously missing pre-bifurcation regime, extending stability guarantees for affine softmax feedback systems to reward-responsive yet globally predictable systems.  It enlarges the certified stability boundary for these systems and identifies where the model genuinely undergoes a phase transition.
\end{abstract}

\section{Introduction}
\label{sec:introduction}

\subsection{Entropy-smoothed response and the stability question}
\label{subsec:intro-broad}

Softmax is often used as a safety mechanism.  In policy learning, game dynamics, variational inference, and population choice, entropy regularization smooths hard best responses and can prevent feedback loops from amplifying small perturbations.  But this safety mechanism has a cost: if the entropy is too large, the policy is stable only because it is too random to strongly follow rewards.  The practically meaningful question is therefore not whether very large entropy stabilizes a softmax system.  It is whether stability can still be proved when the policy has become reward-sensitive and the endogenous feedback loop is genuinely active.

The finite-dimensional model studied here is the affine logit self-consistency equation
\begin{equation}
    x=F_\beta(x):=\sigma\!\bigl(\beta(Wx+b)\bigr),
    \qquad x\in\Delta,
\label{eq:logit-map}
\end{equation}
where $\Delta=\{x\in\mathbb R^n_{\ge0}:\mathbf 1^\top x=1\}$ and $\sigma$ is the softmax operator.  The payoff vector $Wx+b$ is generated by the distribution $x$ itself, and $\beta$ is an inverse temperature, inverse noise level, or rationality parameter.  Continuity already gives existence of a fixed point.  The useful question is quantitative: below what interaction-temperature threshold is the self-consistent distribution unique, and when do natural logit-response dynamics converge globally to it?

This question appears in several literatures under different names.  In quantal-response and population-choice models, it asks when noisy best responses aggregate to a unique equilibrium \cite{mckelvey_palfrey_1995_qre,mckelvey_palfrey_1998_aqre,blume_1993_statmech,alosferrer_netzer_2010_logitresponse,sandholm_2010_pged,BrockDurlauf2001Social}.  In Gibbs, social-interaction, and mean-field variational models, it is the finite-dimensional version of asking when entropy dominates interaction strongly enough to prevent multiple self-consistent phases \cite{Dobrushin1968RandomField,Georgii2011Gibbs,wainwright_jordan_2008_variational,georgii_et_al_2006_potts_gas,mezard_montanari_2009_ipc}.  In logit adjustment and softmax-based learning dynamics, it asks when repeated soft response converges from every initialization \cite{HofbauerSandholm2002StochasticFictitious,alosferrer_netzer_2010_logitresponse,MertikopoulosSandholm2016ReinforcementRegularization,cen_2020_entropy_npg,ding_2021_stochastic_softmax_pg}.  Across these interpretations, the mathematical core is the same: identify the high-temperature stability threshold for entropy-smoothed endogenous response.

\subsection{From a broad question to a quantitative threshold}
\label{subsec:intro-questions}

We organize the paper around a sequence of increasingly concrete questions.

\noindent\textbf{Question 1: Entropy and self-consistency.}
In entropy-smoothed endogenous response systems, when does entropic smoothing enforce a unique and globally stable self-consistent state?

\noindent
We study this question through the finite-dimensional affine logit model, where the abstract stability problem becomes a quantitative spectral problem on the simplex, or on a product of simplices.

\noindent\textbf{Question 2: A spectral high-temperature certificate.}
For affine logit self-consistency, what interaction--temperature condition guarantees global well-posedness, uniqueness, and convergence?

\noindent
This operational question exposes a precise numerical gap in the standard certificate. Classical Euclidean high-temperature arguments yield a dimension-free constant $1$. By contrast, the two-action logit boundary suggests that the correct constant should be $2$.

\noindent\textbf{Question 3: The sharp constant.}
Is the optimal dimension-free Euclidean constant for the global spectral certificate equal to $1$, or is the factor-two improvement to $2$ genuinely attainable?

\noindent
Our answer is that the sharp dimension-free Euclidean constant is \(2\), once interactions are read in the correct tangent-projected geometry.  This closes the factor-two gap between the standard high-temperature certificate and the smallest logit model where the stability boundary can be computed exactly.

The classical unit-softmax analysis proves stability only in an over-regularized regime, where entropy makes the system stable by making responses nearly too random to amplify feedback.  The constant-two theorem certifies the reward-responsive regime: agents can react strongly to payoff differences, yet the endogenous softmax system still has a unique globally attracting state.  In the canonical two-action model, the old certificate stops at \(\beta<1\), while the actual bifurcation occurs at \(\beta=2\).  At that point, a single globally attracting equilibrium splits into multiple stable equilibria and the system becomes path-dependent.  Our theorem reaches this true transition.

\subsection{The classical constant-one certificate and the two-action gap}
\label{subsec:intro-gap}

The standard spectral route uses a coarse estimate of the softmax response.  Treating softmax as a unit-scale map gives
\begin{equation}
    \|\sigma(z)-\sigma(z')\|_2\le \|z-z'\|_2,
\end{equation}
which implies the sufficient condition $\beta\|W\|_2<1$ for contraction of \eqref{eq:logit-map}.  In symmetric potential models, the analogous unit-curvature argument gives a high-temperature condition of the form $\beta\lambda_{\max}(W_{\rm tan})<1$ \cite{gao_pavel_2017_softmax,Dobrushin1968RandomField,mooij_kappen_2007_sumproduct}.  These certificates are useful, but they are conservative.

The smallest symmetric model already exposes the gap.  Consider
\begin{equation}
    W=\begin{pmatrix}0&-1\\-1&0\end{pmatrix},
    \qquad b=0,
\end{equation}
and write $m=x_1-x_2$.  The fixed-point equation becomes
\begin{equation}
    m=\tanh(\beta m/2).
\label{eq:magnetization-main}
\end{equation}
For $\beta\le2$, the only solution is $m=0$.  For every $\beta>2$, two additional nonzero solutions appear.  Thus the canonical boundary is $2$, while the classical unit-scale certificate stops at $1$.  The gap is between a standard certificate and the smallest model where the boundary can be computed exactly.

\subsection{Our answer: covariance-calibrated stability}
\label{subsec:intro-answer}

The missing factor is the covariance geometry of softmax.  If $p=\sigma(z)$, then
\begin{equation}
    D\sigma(z)=\Sigma(p),
    \qquad
    \Sigma(p)=\Diag(p)-pp^\top,
\label{eq:intro-softmax-cov}
\end{equation}
where $\Sigma(p)$ is the covariance matrix of a categorical distribution with law $p$.  Its sharp Euclidean norm is
\begin{equation}
    \sup_{p\in\Delta}\|\Sigma(p)\|_2=\frac12.
\label{eq:intro-cov-half}
\end{equation}
This local response sensitivity is half of what the unit-smoothness proof sees.  Since global contraction asks the response sensitivity to dominate the feedback gain, the reciprocal threshold becomes $2$.

There is one more geometric point.  Perturbations of probability vectors live in the simplex tangent space, and softmax is invariant under adding a constant to all payoffs.  Therefore the effective interaction is not the ambient matrix $W$ but the tangent-payoff quotient
\begin{equation}
    W_{\rm tan}=\Pi W\Pi|_{\Tspace}.
\end{equation}
After this projection, the local covariance bound globalizes along line segments in the simplex and gives
\begin{equation}
    \|F_\beta(x)-F_\beta(y)\|_2
    \le \frac{\beta}{2}\|W_{\rm tan}\|_2\|x-y\|_2.
\end{equation}
The same constant appears in the symmetric variational route, where entropy contributes curvature at least $2$ along every tangent direction.  The contraction, curvature, and pitchfork pictures therefore point to the same boundary.

The covariance identity and the sharp $1/2$ Lipschitz constant are classical and have recently been stated explicitly in tight softmax Lipschitz analyses \cite{nair_2025_softmax_half}.  The contribution here is to turn that local response geometry into a global stability theory for endogenous affine logit systems.

The factor-two threshold has a concrete qualitative meaning in softmax learning and logit population dynamics.  In these systems, the inverse temperature $\beta$ measures how sharply agents react to payoff differences: small $\beta$ makes choices highly random, while large $\beta$ makes choices more responsive and closer to best response.  A high-noise theorem that only works for very small $\beta$ certifies stability mainly in an over-regularized regime where strategic feedback is heavily damped.

The constant-two result certifies a different regime.  In the canonical two-action logit feedback model, the old unit-softmax certificate proves stability only for $\beta<1$, while the actual transition to multiple equilibria occurs at $\beta=2$.  Our theorem reaches this transition: it certifies global stability throughout the entire pre-bifurcation region $\beta<2$.  Thus the improvement changes the certified phase diagram from ``strong noise suppresses feedback'' to ``the whole responsive-but-monostable regime is globally stable.''  Section~\ref{subsec:responsive-monostable-application} explains this application in the main text, and Appendix~\ref{app:responsive-monostable-example} gives the complete worked proof.

\subsection{Contributions and organization}
\label{subsec:intro-contributions}
The paper makes the following contributions.

\paragraph{Constant-two tangent contraction.}
For arbitrary, possibly non-symmetric interactions, after quotienting out payoff shifts and restricting to feasible simplex directions, the condition
\begin{equation}
    \beta\|\Pi W\Pi\|_{\Tspace\to\Tspace}<2
\end{equation}
implies global contraction, a unique fixed point, Picard convergence, and convergence of continuous-time logit adjustment.

\paragraph{Variational constant-two threshold.}
For symmetric interactions, the condition
\begin{equation}
    \beta\lambda_{\max}(\Pi W\Pi|_{\Tspace})<2
\end{equation}
makes the entropy-regularized potential strictly concave, hence gives uniqueness by a variational argument.  The two-action pitchfork realizes the same constant as an actual uniqueness boundary.

\paragraph{Product-simplex extension.}
For multi-population or block-softmax systems with block-specific inverse temperatures, the same constant survives after block tangent projection and temperature scaling:
\begin{equation}
    \|B_\beta\Pi W\Pi\|_{\Tspace\to\Tspace}<2.
\end{equation}
The symmetric product-simplex result has the corresponding block-scaled curvature condition.

The constant-two theorem can replace unit-softmax contraction certificates inside affine softmax-response templates.  Appendix~\ref{sec:consequences-comparison} also gives a signed orthogonal block example where the Euclidean product-simplex certificate succeeds while the natural block $\ell_1$/Dobrushin certificate fails by a factor of order $\sqrt m$.

Section~\ref{sec:setting} defines the simplex and product-simplex geometry. Section~\ref{sec:main-results} states the main theorems. Section~\ref{sec:why-two} gives the technical overview explaining why the number $2$ is unavoidable.  Section~\ref{sec:related-work} discusses related work.

\section{Settings}
\label{sec:setting}

\subsection{Single-simplex logit systems}
\label{subsec:single-setting}

Let $\Delta=\{x\in\mathbb R^n_{\ge0}:\mathbf 1^\top x=1\}$. For a matrix $W\in\mathbb R^{n\times n}$, a vector $b\in\mathbb R^n$, and an inverse temperature $\beta>0$, the single-population logit response map is
\begin{equation}
    F_\beta(x)=\sigma\bigl(\beta(Wx+b)\bigr),
    \qquad x\in\Delta.
\end{equation}
Here $x$ can be read as a probability distribution, a population share vector, or a block of variational marginals.  The affine payoff $Wx+b$ contains exogenous utility $b$ and endogenous feedback $Wx$.  A fixed point of $F_\beta$ is a self-consistent smoothed response: the distribution induced by the payoff is the same distribution that generated the payoff.

\subsection{Feasible perturbations, payoff shifts, and tangent projection}
\label{subsec:tangent-geometry}

Two invariances determine the correct geometry.  First, differences of feasible probability vectors lie in the simplex tangent space
\begin{equation}
    \Tspace=\{v\in\mathbb R^n:\mathbf 1^\top v=0\}.
\end{equation}
Second, softmax is invariant under payoff shifts:
\begin{equation}
    \sigma(z+c\mathbf 1)=\sigma(z),
    \qquad c\in\mathbb R.
\end{equation}
Thus perturbations live in $\Tspace$, while payoff perturbations are observed only modulo the all-ones direction.

Let $\Pi=I-\frac1n\mathbf 1\mathbf 1^\top$ be the orthogonal projection onto $\Tspace$.  The effective interaction operator is
\begin{equation}
    \Wt:=\Pi W\Pi\big|_{\Tspace}:\Tspace\to\Tspace.
\label{eq:def-tangent-W}
\end{equation}
The projection is the quotient geometry forced by feasible probability perturbations and payoff-shift invariance.  Throughout the paper, $\Wt$ denotes the tangent interaction operator; $W^\top$ is reserved for the transpose.

\subsection{Product-simplex block logit systems}
\label{subsec:product-setting}

Many applications involve several populations, players, agents, or coordinate blocks.  Let
\begin{equation}
    \DeltaProd=\Delta_1\times\cdots\times\Delta_m,
    \qquad
    \Tspace=\Tspace_1\oplus\cdots\oplus\Tspace_m,
\end{equation}
where $\Delta_a\subset\mathbb R^{n_a}$ and $\Tspace_a=\{v^a:\mathbf 1^\top v^a=0\}$.  Let $\Pi=\blkdiag(\Pi_1,\ldots,\Pi_m)$ be the block tangent projection.  With block matrices $W^{ab}\in\mathbb R^{n_a\times n_b}$ and block biases $b^a$, define
\begin{equation}
    F^a(x)=\sigma_a\!\left(\beta_a\left(b^a+\sum_{c=1}^m W^{ac}x^c\right)\right),
    \qquad a=1,\ldots,m,
\label{eq:product-map}
\end{equation}
where each block has its own inverse temperature $\beta_a>0$.  We write
\begin{equation}
    B_\beta=\blkdiag(\beta_1 I_{n_1},\ldots,\beta_m I_{n_m}).
\end{equation}
The product-simplex setting is a direct model of multi-population logit equilibrium, multi-agent soft response, and block-softmax variational equations.

\subsection{Two proof routes and the softmax covariance lemma}
\label{subsec:two-routes}

There are two routes to uniqueness.  For arbitrary interactions, no potential is available, so we prove response-map contraction.  For symmetric interactions, a potential exists, so we prove strict concavity by comparing entropy curvature to interaction curvature.  The same two-point geometry underlies both routes.

\begin{lemma}[Softmax covariance geometry]
\label{lem:softmax-geometry}\label{lem:jac_factor}\label{lem:cov_bound}
For $p=\sigma(z)$,
\begin{equation}
    D\sigma(z)=\Sigma(p),\qquad \Sigma(p)=\Diag(p)-pp^\top.
\end{equation}
Consequently, for $F_\beta(x)=\sigma(\beta(Wx+b))$,
\begin{equation}
    DF_\beta(x)=\beta\Sigma(F_\beta(x))W,
    \qquad
    \mathbf 1^\top DF_\beta(x)=0^\top.
\end{equation}
Moreover,
\begin{equation}
    \|\Sigma(p)\|_2\le\frac12\quad\text{for all }p\in\Delta,
\end{equation}
and the constant $1/2$ is attained by a distribution supported equally on two coordinates.
\end{lemma}

\section{Main results}
\label{sec:main-results}

\subsection{Arbitrary interactions: tangent spectral contraction}
\label{subsec:general-main-result}

\begin{theorem}[Single-simplex tangent contraction and dynamics]
\label{thm:main}\label{thm:NS_contraction}\label{cor:A_ode_global}
Consider the logit self-map $F_\beta(x)=\sigma(\beta(Wx+b))$ on $\Delta$.  If
\begin{equation}
    q:=\frac{\beta}{2}\|\Wt\|_2<1,
\end{equation}
then $F_\beta$ is a contraction on $(\Delta,\|\cdot\|_2)$ with factor $q$.  Hence it has a unique fixed point $x^\star$, and the Picard iteration $x^{k+1}=F_\beta(x^k)$ satisfies
\begin{equation}
    \|x^k-x^\star\|_2\le q^k\|x^0-x^\star\|_2,
    \qquad
    \|x^k-x^\star\|_2\le\frac{\|x^{k+1}-x^k\|_2}{1-q}.
\end{equation}
The continuous-time logit adjustment dynamics
\begin{equation}
    \dot x(t)=F_\beta(x(t))-x(t)
\end{equation}
converge to the same fixed point with rate $e^{-(1-q)t}$.
\end{theorem}

The result allows arbitrary non-symmetric feedback.  No monotonicity, positivity, or potential structure is required.  Stability is certified purely by the covariance-damped tangent feedback gain.

\subsection{Symmetric interactions: entropy-regularized potential uniqueness}
\label{subsec:symmetric-main-result}

\begin{theorem}[Single-simplex entropy-curvature threshold]
\label{thm:C3_unique_fp}
Suppose $W=W^\top$, let $\kappa=\lambda_{\max}(\Wt)$, and define
\begin{equation}
    \Phi_\beta(x)=H(x)+\frac\beta2x^\top Wx+\beta b^\top x,
    \qquad
    H(x)=-\sum_i x_i\log x_i.
\label{eq:Phi_def}
\end{equation}
If either $\kappa\le0$ or $\beta\kappa<2$, then $\Phi_\beta$ is strictly concave on the affine hull of the simplex.  Its unique maximizer lies in $\Delta^\circ$ and is the unique fixed point of $F_\beta$.
\end{theorem}

The same constant appears in the variational route because entropy has curvature at least $2$ on the simplex tangent space.  In this route, the threshold is an entropy-versus-interaction curvature balance rather than a response-map contraction condition.

\subsection{Product-simplex systems: blockwise constant two}
\label{subsec:product-main-result}

\begin{theorem}[Product-simplex tangent contraction]
\label{thm:product-contraction}
Consider the block logit map \eqref{eq:product-map} on $\DeltaProd$.  If
\begin{equation}
    q_{\mathrm{prod}}:=\frac12\bigl\|B_\beta\Pi W\Pi\bigr\|_{\Tspace\to\Tspace}<1,
\end{equation}
then $F$ is a global contraction on $(\DeltaProd,\|\cdot\|_2)$ with factor $q_{\mathrm{prod}}$.  Hence the product-simplex logit system has a unique fixed point; block Picard iteration and continuous-time block logit adjustment converge globally.
\end{theorem}

\begin{theorem}[Product-simplex entropy-curvature threshold]
\label{thm:product-potential}
Suppose $W=W^\top$ and consider the product-simplex potential
\begin{equation}
    \Psi(x)=\sum_{a=1}^m \frac{1}{\beta_a}H(x^a)+\frac12x^\top Wx+b^\top x.
\label{eq:product-potential}
\end{equation}
If
\begin{equation}
    \lambda_{\max}\!\left(
        B_\beta^{1/2}\Pi W\Pi B_\beta^{1/2}\big|_{\Tspace}
    \right)<2,
\end{equation}
then $\Psi$ is strictly concave on the affine hull of $\DeltaProd$.  Its unique maximizer lies in the relative interior and is the unique fixed point of \eqref{eq:product-map}.
\end{theorem}

The product-simplex result shows that the constant does not degrade with the number of populations: once every interaction is read on the block tangent space and scaled by its block temperature, the same two-point softmax covariance bound controls the whole coupled system.

\subsection{Sharp boundary witness: the two-action pitchfork}
\label{subsec:pitchfork-main-result}

\begin{proposition}[Two-action boundary]
\label{prop:pitchfork-boundary}\label{thm:F1_sharpness_pitchfork}
For
\begin{equation}
    W=\begin{pmatrix}0&-1\\-1&0\end{pmatrix},
    \qquad b=0,
\end{equation}
one has $\lambda_{\max}(\Wt)=1$.  The fixed point is unique for $\beta\le2$.  For every $\beta>2$, exactly two additional non-uniform fixed points appear.  Thus uniqueness is lost exactly at $\beta\lambda_{\max}(\Wt)=2$ in the canonical two-action potential example.
\end{proposition}

This matches the upper boundary suggested by the two-action model.  Classical unit-scale certificates reach $1$; the canonical two-action boundary is $2$; the covariance-calibrated theorem reaches $2$.

\begin{table}[t]
\centering
\caption{
\textbf{Closing the constant-one/constant-two gap in finite-dimensional logit self-consistency.}
The ``Before'' bounds reflect standard contraction or high-temperature criteria from coarse Lipschitz or curvature estimates \cite{gao_pavel_2017_softmax,Dobrushin1968RandomField,mooij_kappen_2007_sumproduct}.  The ``After'' bounds are the covariance-calibrated tangent thresholds proved here.  Here $\Pi$ always denotes the appropriate tangent projection, and $\Wt:=\Pi W\Pi|_{\Tspace}$ avoids any ambiguity with the transpose $W^\top$.
}
\label{tab:before-after}
\resizebox{\textwidth}{!}{%
\begin{tabular}{lll}
\toprule
Regime & Before & After \\
\midrule
single simplex, arbitrary $W$
  & $\beta\|W\|_2<1$
  & $\beta\|\Wt\|_2<2$ \\
single simplex, $W=W^\top$
  & $\beta\lambda_{\max}(\Wt)<1$
  & $\beta\lambda_{\max}(\Wt)<2$ \\
product simplex, arbitrary blocks
  & unit-scale / Dobrushin-style certificates
  & $\|B_\beta \Pi W\Pi\|_2<2$ \\
product simplex, symmetric blocks
  & unit-scale curvature certificate
  & $\lambda_{\max}(B_\beta^{1/2}\Pi W\Pi B_\beta^{1/2})<2$ \\
\bottomrule
\end{tabular}%
}
\end{table}

\section{Why the constant is two: technical overview}
\label{sec:why-two}

The proofs are short once the geometry is set up.  The technical story is best understood as a sequence of corrections to the naive high-temperature proof: the response scale is $1/2$; the interaction must be projected to the feasible payoff quotient; block systems inherit the same covariance bound; and the symmetric route sees the same constant through entropy curvature.

A conservative proof begins with
\begin{equation}
    \|\sigma(z)-\sigma(z')\|_2\le\|z-z'\|_2.
\end{equation}
Applying it to $z=\beta(Wx+b)$ and $z'=\beta(Wy+b)$ gives
\begin{equation}
    \|F_\beta(x)-F_\beta(y)\|_2
    \le \beta\|W\|_2\|x-y\|_2.
\end{equation}
This proves contraction only under $\beta\|W\|_2<1$.  Two losses are hidden in this one line.  First, the local Euclidean sensitivity of softmax is $1/2$, not $1$.  Second, the ambient norm of $W$ counts payoff-shift and infeasible directions that cannot affect the simplex response.  The corrected proof replaces unit smoothness by covariance geometry and $W$ by $\Pi W\Pi|_{\Tspace}$.

The derivative of softmax is the categorical covariance matrix:
\begin{equation}
    D\sigma(z)=\Sigma(p),
    \qquad p=\sigma(z),
\end{equation}
and
\begin{equation}
    v^\top\Sigma(p)v
    =\operatorname{Var}_{i\sim p}(v_i).
\end{equation}
The largest possible variance of the coordinates of a unit vector, under an arbitrary categorical distribution, is $1/2$.  It is attained by placing mass $1/2$ and $1/2$ on two coordinates with opposite signs.  This two-point extremizer is already visible in the pitchfork model, which is why the same constant appears as an actual boundary there.

For $h\in\Tspace$, payoff-shift invariance and tangent feasibility give
\begin{equation}
    DF_\beta(x)h
    =\beta\Sigma(F_\beta(x))\Pi W\Pi h.
\end{equation}
Thus
\begin{equation}
    \|DF_\beta(x)|_{\Tspace}\|_2
    \le \frac\beta2\|\Pi W\Pi\|_{\Tspace\to\Tspace}.
\end{equation}
For $x,y\in\Delta$, the segment $y+t(x-y)$ remains in the simplex and $x-y\in\Tspace$, so the fundamental theorem of calculus gives
\begin{equation}
    F_\beta(x)-F_\beta(y)
    =\int_0^1DF_\beta(y+t(x-y))(x-y)\,dt.
\end{equation}
Consequently,
\begin{equation}
    \|F_\beta(x)-F_\beta(y)\|_2
    \le \frac\beta2\|W_{\rm tan}\|_2\|x-y\|_2.
\end{equation}
The contraction threshold is exactly the condition that this product is below one.

For product simplexes, the Jacobian has a block covariance factor:
\begin{equation}
    DF(x)=\Sigma_{\rm blk}(x)B_\beta W,
\end{equation}
where $\Sigma_{\rm blk}(x)$ is block diagonal and each block is a softmax covariance matrix.  After projecting to the block tangent space,
\begin{equation}
    DF(x)h=\Sigma_{\rm blk}(x)B_\beta\Pi W\Pi h,
    \qquad h\in\Tspace.
\end{equation}
Because a block diagonal matrix has norm equal to the maximum of its block norms, and each softmax covariance block has norm at most $1/2$, the global block system still has the same response scale:
\begin{equation}
    \|DF(x)|_{\Tspace}\|_2
    \le \frac12\|B_\beta\Pi W\Pi\|_{\Tspace\to\Tspace}.
\end{equation}
The constant does not deteriorate with the number of populations.

For symmetric $W$, the fixed points are KKT points of the entropy-regularized potential \eqref{eq:Phi_def}.  Along a tangent direction $v\in\Tspace$,
\begin{equation}
    v^\top\nabla^2H(x)v=-\sum_i\frac{v_i^2}{x_i}.
\end{equation}
The key tangent inequality is
\begin{equation}
    \sum_i\frac{v_i^2}{x_i}\ge2\|v\|_2^2,
    \qquad v\in\Tspace.
\label{eq:tangent-entropy-two-main}
\end{equation}
This is the curvature counterpart of the covariance norm bound.  Intuitively, a zero-sum perturbation must move probability mass from a positive side to a negative side, and the least-curved transfer is two-dimensional.

The quadratic interaction contributes at most $\beta\lambda_{\max}(\Wt)\|v\|_2^2$.  Hence
\begin{equation}
    v^\top\nabla^2\Phi_\beta(x)v
    \le -(2-\beta\lambda_{\max}(\Wt))\|v\|_2^2.
\end{equation}
Strict concavity follows when $\beta\lambda_{\max}(\Wt)<2$.  The product-simplex potential is handled the same way after rescaling each block by $\beta_a^{1/2}$.

The constant can be read in three equivalent ways.  In the contraction route, local response sensitivity is $1/2$ and feedback gain is $\beta\|W_{\rm tan}\|$, so contraction asks
\begin{equation}
    \frac12\beta\|W_{\rm tan}\|<1.
\end{equation}
In the curvature route, entropy contributes tangent curvature $2$, while interaction contributes at most $\beta\kappa$.  In the pitchfork model, the scalar response map has slope $\beta/2$ at the uniform state.  All three calculations change stability at the same number: $2$.

\subsection{Application: certifying the responsive-but-monostable regime}
\label{subsec:responsive-monostable-application}

We now spell out what the constant improvement means in a simple but representative logit feedback system.  Consider a population, or a stateless softmax learner, choosing between two actions.  The state is the population distribution $x\in\Delta_2$, and the payoff vector is generated by the current state itself:
\begin{equation}
    F_\beta(x)=\softmax(\beta Wx),
    \qquad
    W=\begin{pmatrix}0&-1\\-1&0\end{pmatrix}.
\end{equation}
This is the smallest model in which softmax response and endogenous feedback interact nontrivially.  The parameter $\beta$ is the inverse noise level.  Small $\beta$ means that choices are almost random; large $\beta$ means that agents react strongly to payoff differences.

The classical unit-softmax certificate gives stability only for
\begin{equation}
    \beta<1.
\end{equation}
This says that the system is stable when responses are sufficiently noisy.  The constant-two theorem gives
\begin{equation}
    \beta<2.
\end{equation}
In this interval, agents are already substantially responsive to payoff differences, and the feedback loop is strong enough to matter.  Nevertheless, the system remains monostable: there is still only one equilibrium, and every trajectory converges to it.

\begin{proposition}[Certified responsive-but-monostable regime]
\label{prop:worked-two-action-phase-diagram}
For the two-action logit feedback system
\begin{equation}
    F_\beta(x)=\softmax(\beta Wx),
    \qquad
    W=\begin{pmatrix}0&-1\\-1&0\end{pmatrix},
\end{equation}
the tangent feedback norm satisfies $\|\Pi W\Pi\|_{\Tspace\to\Tspace}=1$.  Hence the classical unit-softmax contraction route certifies global stability only for $\beta<1$, while Theorem~\ref{thm:main} certifies global uniqueness and global exponential convergence for every $\beta<2$.

Moreover, this threshold is the true bifurcation threshold of the model.  Writing
\begin{equation}
    m=x_1-x_2,
\end{equation}
the continuous-time logit adjustment dynamics reduce to
\begin{equation}
    \dot m=\tanh(\beta m/2)-m.
\end{equation}
For $\beta<2$, the unique equilibrium is $m=0$, and every trajectory converges to it.  For $\beta>2$, the equilibrium $m=0$ becomes unstable and two additional stable nonzero equilibria appear.  Thus the interval $1\le\beta<2$ is a genuinely responsive but still globally predictable regime that the unit-softmax analysis leaves uncertified.
\end{proposition}

The simplest exactly solvable logit feedback model exposes what the constant means.  The old certificate covers only a strongly damped high-noise regime.  The covariance-calibrated certificate reaches the actual pre-bifurcation boundary.  In applications, this means that the theorem can certify stable softmax behavior when agents or policies are already meaningfully reward-sensitive, rather than only when entropy has made them nearly random.
The exact scalar phase diagram behind Proposition~\ref{prop:worked-two-action-phase-diagram} is visualized in Appendix~\ref{app:num-two-action-phase-diagram}.

\section{Related work}
\label{sec:related-work}

\subsection{Softmax sensitivity and high-temperature certificates}
\label{subsec:rw-softmax-high-temp}

The softmax Jacobian covariance identity is classical, and the exact Lipschitz behavior of softmax has recently been isolated explicitly: \citet{nair_2025_softmax_half} proves tight norm-uniform bounds, including the sharp Euclidean constant $1/2$.  Earlier convex-analytic work on softmax and regularized best response developed monotonicity and nonexpansiveness properties useful in game-theoretic and reinforcement-learning settings \cite{gao_pavel_2017_softmax}.  High-temperature uniqueness results in Gibbs measures, correlation decay, variational inference, and Bethe/free-energy methods likewise control when interaction is dominated by entropy or local conditional sensitivity \cite{Dobrushin1968RandomField,Georgii2011Gibbs,dobrushin_shlosman_1985_completely,weitz_2005_uniqueness,hayes_2006_rapid_mixing,martinelli_1999_glauber,JordanGhahramaniJaakkolaSaul1999Variational,wainwright_jordan_2008_variational,yedidia_freeman_weiss_2005_free_energy,heskes_2004_uniqueness_lbp,mooij_kappen_2007_sumproduct}.  These criteria are closest in spirit to ours, but they usually control coordinate or block influence, often in $\ell_1$-type geometries; our certificate instead uses the signed Euclidean tangent operator induced by simplex feasibility and payoff-shift invariance.

\subsection{Logit equilibrium and endogenous response}
\label{subsec:rw-logit-equilibrium}

Quantal-response equilibrium defines equilibria as fixed points of noisy best-response maps \cite{mckelvey_palfrey_1995_qre,mckelvey_palfrey_1998_aqre,goeree_holt_palfrey_2008_newpalgrave,goeree_holt_palfrey_2016_book}.  Aggregate logit, social-interaction, and network choice models study self-consistent population choice, identification, and the emergence or uniqueness of multiple equilibria \cite{blume_1993_statmech,BrockDurlauf2001Social,haile_hortacsu_kosenok_2008_empirical_content,Melo2022UniquenessQRE}.  We focus on a dimension-free spectral certificate for finite-dimensional affine logit self-consistency after projecting onto the payoff-relevant tangent geometry.

\subsection{Logit dynamics, bifurcation, and entropy-regularized learning}
\label{subsec:rw-logit-dynamics}

Smoothed best-response, logit-response, perturbed-payoff, and regularized learning dynamics are standard in games, evolutionary dynamics, and stochastic approximation \cite{HofbauerSandholm2002StochasticFictitious,hofbauer_sandholm_2007_disturbed_payoffs,alosferrer_netzer_2010_logitresponse,sandholm_2010_pged,fudenberg_levine_1998_learning,benaim_1996_dynamical_system,benaim_1999_dynamics_sa,benaim_hofbauer_sorin_2005_sadi,benaim_hofbauer_sorin_2006_sadi2,MertikopoulosSandholm2016ReinforcementRegularization}.  A related line studies high- and low-noise stability, pitchfork bifurcations, multiple equilibria, limit cycles, metastability, positive feedback, and softmax policy-gradient dynamics in game-specific models \cite{cianfanelli_como_2025_logit_stability,gavin_cao_paarporn_2024_rlambert,hommes_ochea_2012_logit_dynamics,umezuki_2018_rps_logit,hwang_reybellet_2021_positive_feedback,ferraioli_ventre_2019_metastability,leung_hu_leung_2024_softmax_pg_games,li_2025_logit_softmax_pg}.  Entropy/KL regularization and softmax policies are also central in stochastic control and RL \cite{Kappen2005PathIntegrals,todorov_2006_lmdp,todorov_2009_optimal_actions,ziebart_2010_maxentirl,NeuJonssonGomez2017UnifiedEntropyMDP,GeistScherrerPietquin2019RegularizedMDP}.  Our results are complementary: they give deterministic exact-response fixed-point certificates for uniqueness, global contraction, and strict concavity; they do not by themselves establish stochastic-approximation convergence unless the relevant analysis reduces to the same affine contraction property.

\section{Conclusion}

This paper gives a sharper stability guarantee for softmax feedback systems.  Existing high-temperature analyses certify uniqueness and convergence only when the entropy is so strong that the softmax response is close to random.  We show that this is unnecessarily conservative: after using the exact covariance sensitivity of softmax and the tangent geometry of the simplex, the certified stable region doubles relative to the coarse unit-softmax certificate. In the canonical two-action example, it reaches the true point where a single stable fixed point splits into multiple equilibria.

\bibliographystyle{plainnat}
\bibliography{reference}

\appendix

\paragraph{Appendix roadmap.}
The appendices are organized proof-first. Appendix~\ref{app:softmax-geometry} collects the simplex geometry and softmax covariance facts. Appendices~\ref{app:general-proofs}--\ref{app:sharpness-proofs} prove the single-simplex, product-simplex, symmetric, and two-action sharpness results. Appendix~\ref{app:responsive-monostable-example} gives the complete worked two-action proof. Appendix~\ref{sec:consequences-comparison} records secondary consequences and comparisons, with proofs in Appendix~\ref{app:consequences-comparison}. Appendix~\ref{app:numerical-illustrations} contains deterministic numerical illustrations, and Appendix~\ref{sec:discussion} states sharpness qualifications.

\section{Preliminaries and softmax geometry}
\label{app:softmax-geometry}

This appendix proves the softmax facts used in the main text.  The notation is as follows: $\mathbf 1$ is the all-ones vector, $\Tspace=\{v:\mathbf 1^\top v=0\}$ in the single-simplex case, and $\Pi=I-n^{-1}\mathbf 1\mathbf 1^\top$ is the orthogonal projection onto $\Tspace$.

\begin{proof}[Proof of Lemma~\ref{lem:jac_factor}]
For $\sigma(z)_i=e^{z_i}/\sum_j e^{z_j}$ and $p=\sigma(z)$,
\begin{equation}
    \frac{\partial \sigma_i}{\partial z_j}(z)=p_i(\mathbf 1\{i=j\}-p_j),
    \qquad
    D\sigma(z)=\Diag(p)-pp^\top=\Sigma(p).
\end{equation}
With $z=\beta(Wx+b)$, the chain rule gives $DF_\beta(x)=\beta\Sigma(F_\beta(x))W$.  Since $\mathbf 1^\top\Sigma(p)=0^\top$, we have $\mathbf 1^\top DF_\beta(x)=0^\top$, so every derivative output lies in $\Tspace$.

For any unit vector $v$,
\begin{equation}
    v^\top\Sigma(p)v
    =\sum_i p_i v_i^2-\Bigl(\sum_i p_i v_i\Bigr)^2
    =\Var_{i\sim p}(v_i)
    \le \frac{(\max_i v_i-\min_i v_i)^2}{4}
    \le \frac{\|v\|_2^2}{2}
    =\frac12.
\end{equation}
The penultimate inequality follows from $(v_i-v_j)^2\le2(v_i^2+v_j^2)\le2\|v\|_2^2$.  Equality is attained by $p=(1/2,1/2,0,\ldots,0)$ and $v=(1,-1,0,\ldots,0)/\sqrt2$.
\end{proof}

\section{Proofs for single-simplex general interactions}
\label{app:general-proofs}

\begin{proof}[Proof of the contraction statement in Theorem~\ref{thm:NS_contraction}]
For $h\in\Tspace$, $h=\Pi h$.  Also $\Sigma(p)\Pi=\Sigma(p)$ because $\Sigma(p)\mathbf 1=0$.  Hence
\begin{equation}
    DF_\beta(x)h
    =\beta\Sigma(F_\beta(x))W h
    =\beta\Sigma(F_\beta(x))\Pi W\Pi h.
\end{equation}
By Lemma~\ref{lem:jac_factor},
\begin{equation}
    \sup_{z\in\Delta}\|DF_\beta(z)|_{\Tspace}\|_2
    \le\frac\beta2\|\Wt\|_2=q.
\end{equation}
For $x,y\in\Delta$, set $h=x-y\in\Tspace$.  The segment $y+th$ remains in $\Delta$, and the integral formula gives
\begin{equation}
    F_\beta(x)-F_\beta(y)=\int_0^1 DF_\beta(y+th)h\,dt.
\end{equation}
Thus $\|F_\beta(x)-F_\beta(y)\|_2\le q\|x-y\|_2$.  If $q<1$, Banach's fixed-point theorem on the complete metric space $\Delta$ gives existence, uniqueness, and the Picard rate.  The a posteriori estimate follows by summing the contraction tail:
\begin{equation}
    \|x^k-x^\star\|_2
    \le\sum_{j=k}^{\infty}\|x^{j+1}-x^j\|_2
    \le\sum_{r=0}^{\infty}q^r\|x^{k+1}-x^k\|_2
    =\frac{\|x^{k+1}-x^k\|_2}{1-q}.
\end{equation}
\end{proof}

\begin{proof}[Proof of the ODE statement in Theorem~\ref{thm:NS_contraction}]
The vector field is smooth and tangent to $\Delta$, hence trajectories starting in $\Delta$ remain in $\Delta$.  Let $y(t)=x(t)-x^\star$.  For $y(t)\ne0$,
\begin{equation}
    \frac{d}{dt}\|y(t)\|_2
    =\Bigl\langle\frac{y(t)}{\|y(t)\|_2},F_\beta(x(t))-F_\beta(x^\star)-y(t)\Bigr\rangle
    \le q\|y(t)\|_2-\|y(t)\|_2
    =-(1-q)\|y(t)\|_2.
\end{equation}
Gronwall's inequality gives the claimed exponential bound; the case $y(t)=0$ is immediate.
\end{proof}

\section{Proofs for product-simplex general interactions}
\label{app:product-proofs}

Let $\Tspace=\oplus_a\Tspace_a$ and $\Pi=\blkdiag(\Pi_1,\ldots,\Pi_m)$.  For $x\in\DeltaProd$, define
\begin{equation}
    \Sigma_{\mathrm{blk}}(x)=\blkdiag(\Sigma_1(F^1(x)),\ldots,\Sigma_m(F^m(x))).
\end{equation}
Then $\|\Sigma_{\mathrm{blk}}(x)\|_2\le1/2$ because it is block diagonal and each block is a softmax covariance matrix.

\begin{proof}[Proof of the contraction statement in Theorem~\ref{thm:product-contraction}]
For $h\in\Tspace$, $h=\Pi h$.  The block Jacobian satisfies
\begin{equation}
    DF(x)=\Sigma_{\mathrm{blk}}(x)B_\beta W.
\end{equation}
Since each covariance block annihilates the corresponding all-ones direction and $B_\beta$ is scalar on each block,
\begin{equation}
    DF(x)h
    =\Sigma_{\mathrm{blk}}(x)B_\beta\Pi W\Pi h.
\end{equation}
Therefore
\begin{equation}
    \|DF(x)|_{\Tspace}\|_2
    \le \frac12\|B_\beta\Pi W\Pi\|_{\Tspace\to\Tspace}=q_{\mathrm{prod}}.
\end{equation}
The same integral argument on the convex product simplex gives
\begin{equation}
    \|F(x)-F(y)\|_2\le q_{\mathrm{prod}}\|x-y\|_2.
\end{equation}
If $q_{\mathrm{prod}}<1$, Banach's theorem gives a unique fixed point and convergence of Picard iteration.  The continuous-time block logit adjustment proof is identical to the single-simplex ODE proof with $q$ replaced by $q_{\mathrm{prod}}$.
\end{proof}

\section{Proofs for symmetric interaction matrices}
\label{app:symmetric-proofs}

Throughout this appendix the matrix under discussion is symmetric, and all eigenvalues are computed on the corresponding tangent space.

\subsection{Single-simplex potential proof}

On $\Delta^\circ$,
\begin{equation}
    \nabla\Phi_\beta(x)_i=-1-\log x_i+\beta(Wx+b)_i,
    \qquad
    \nabla^2\Phi_\beta(x)=-\Diag(1/x)+\beta W.
\label{eq:app-grad-hess-phi}
\end{equation}

\paragraph{Interior maximizers and KKT equations.}
Every global maximizer of $\Phi_\beta$ over $\Delta$ lies in $\Delta^\circ$.  Indeed, if $x_j=0$ and $x_i>0$, then for $x(\varepsilon)=x+\varepsilon e_j-\varepsilon e_i$,
\begin{equation}
    \Phi_\beta(x(\varepsilon))-\Phi_\beta(x)
    =-\varepsilon\log\varepsilon+O(\varepsilon)>0
\end{equation}
for all sufficiently small $\varepsilon>0$; the quadratic and linear terms contribute only $O(\varepsilon)$.

For $x\in\Delta^\circ$, the KKT condition for maximizing $\Phi_\beta$ over the affine simplex is that $\nabla\Phi_\beta(x)=\lambda\mathbf 1$ for some $\lambda\in\mathbb R$.  By \eqref{eq:app-grad-hess-phi}, this is equivalent to
\begin{equation}
    -1-\log x_i+\beta(Wx+b)_i=\lambda
    \quad\Longleftrightarrow\quad
    x_i=\exp(\beta(Wx+b)_i-1-\lambda).
\end{equation}
Normalizing over $i$ gives $x=\sigma(\beta(Wx+b))$.  Conversely, the softmax equation gives the KKT condition.  Thus fixed points are exactly interior KKT points of $\Phi_\beta$.

\paragraph{Entropy curvature on the tangent space.}
For $x\in\Delta^\circ$ and $v\in\Tspace$,
\begin{equation}
    v^\top\nabla^2H(x)v=-\sum_i\frac{v_i^2}{x_i}.
\end{equation}
The key lower bound is
\begin{equation}
    \sum_i\frac{v_i^2}{x_i}\ge2\|v\|_2^2,
    \qquad v\in\Tspace.
\label{eq:app-tangent-two}
\end{equation}
To prove it, write $P=\{i:v_i\ge0\}$, $N=\{i:v_i<0\}$, $a=\sum_{i\in P}v_i=-\sum_{i\in N}v_i$, and $s=\sum_{i\in P}x_i$.  Cauchy--Schwarz gives
\begin{equation}
    \sum_i\frac{v_i^2}{x_i}
    \ge \frac{a^2}{s}+\frac{a^2}{1-s}
    \ge4a^2.
\end{equation}
Since the positive and negative parts of $v$ each have $\ell_1$ mass $a$,
\begin{equation}
    \|v\|_2^2=\|v_+\|_2^2+\|v_-\|_2^2\le a^2+a^2=2a^2.
\end{equation}
Combining the two inequalities yields \eqref{eq:app-tangent-two}.  Equality holds for $n=2$, $x=(1/2,1/2)$, and $v=(1,-1)$.

\begin{proof}[Proof of the symmetric statement in Theorem~\ref{thm:C3_unique_fp}]
For $x\in\Delta^\circ$ and $v\in\Tspace$, \eqref{eq:app-grad-hess-phi} and \eqref{eq:app-tangent-two} give
\begin{equation}
    v^\top\nabla^2\Phi_\beta(x)v
    =-\sum_i\frac{v_i^2}{x_i}+\beta v^\top Wv
    \le -2\|v\|_2^2+\beta\kappa\|v\|_2^2
    =-(2-\beta\kappa)\|v\|_2^2.
\end{equation}
If $\kappa\le0$, this is at most $-2\|v\|_2^2$; if $\kappa>0$ and $\beta\kappa<2$, it is strictly negative with modulus $2-\beta\kappa$.  Hence $\Phi_\beta$ is strictly concave on every feasible segment in $\Delta$.

Compactness gives a global maximizer, the boundary argument above makes it interior, and strict concavity makes it unique.  The KKT equivalence above identifies this maximizer with a logit fixed point.  Conversely, any fixed point is an interior KKT point; for a concave differentiable function, such a point is a global maximizer because for every $y\in\Delta$,
\begin{equation}
    \Phi_\beta(y)\le\Phi_\beta(x)+\langle\nabla\Phi_\beta(x),y-x\rangle=\Phi_\beta(x).
\end{equation}
Strict concavity therefore rules out any second fixed point.
\end{proof}

\subsection{Product-simplex potential proof}

\begin{proof}[Proof of the symmetric product statement in Theorem~\ref{thm:product-potential}]
On the relative interior of $\DeltaProd$, the gradient of \eqref{eq:product-potential} in block $a$ is
\begin{equation}
    \nabla_{x^a}\Psi(x)_i=-\frac{1}{\beta_a}(1+\log x_i^a)+(Wx+b)^a_i.
\end{equation}
The KKT condition on each block is therefore
\begin{equation}
    -\frac{1}{\beta_a}(1+\log x_i^a)+(Wx+b)^a_i=\lambda_a,
\end{equation}
which is equivalent, after exponentiating and normalizing within block $a$, to
\begin{equation}
    x^a=\sigma_a\left(\beta_a (Wx+b)^a\right).
\end{equation}
Thus fixed points of \eqref{eq:product-map} are exactly interior KKT points of $\Psi$.  As in the single-simplex case, the entropy term sends every maximizer into the relative interior of each block.

For $v\in\Tspace$,
\begin{equation}
    v^\top\nabla^2\Psi(x)v
    =-\sum_{a=1}^m\frac{1}{\beta_a}\sum_i\frac{(v_i^a)^2}{x_i^a}+v^\top Wv.
\end{equation}
By applying \eqref{eq:app-tangent-two} blockwise,
\begin{equation}
    v^\top\nabla^2\Psi(x)v
    \le -2\sum_{a=1}^m\frac{1}{\beta_a}\|v^a\|_2^2+v^\top Wv.
\end{equation}
Let $y=B_\beta^{-1/2}v$.  Since $B_\beta$ is scalar on each block, $y\in\Tspace$ and $v=B_\beta^{1/2}y$.  Hence
\begin{equation}
    v^\top\nabla^2\Psi(x)v
    \le -2\|y\|_2^2
    +y^\top B_\beta^{1/2}\Pi W\Pi B_\beta^{1/2}y.
\end{equation}
If the largest tangent eigenvalue of $B_\beta^{1/2}\Pi W\Pi B_\beta^{1/2}$ is strictly smaller than $2$, the Hessian is strictly negative on every nonzero feasible direction.  Therefore $\Psi$ is strictly concave on the affine hull of $\DeltaProd$.  The same compactness, interiority, and KKT argument used above gives a unique maximizer, and this maximizer is the unique block-logit fixed point.
\end{proof}

\section{Sharpness via the two-action pitchfork}
\label{app:sharpness-proofs}

\begin{proof}[Proof of Proposition~\ref{prop:pitchfork-boundary}]
Write
\begin{equation}
    x(m)=\left(\frac{1+m}{2},\frac{1-m}{2}\right),
    \qquad m\in[-1,1].
\end{equation}
For $W=\left(\begin{smallmatrix}0&-1\\-1&0\end{smallmatrix}\right)$,
\begin{equation}
    Wx(m)=\left(-\frac{1-m}{2},-\frac{1+m}{2}\right),
    \qquad
    \beta[(Wx)_1-(Wx)_2]=\beta m.
\end{equation}
Thus
\begin{equation}
    F_\beta(x(m))_1-F_\beta(x(m))_2
    =\frac{e^{\beta(Wx)_1}-e^{\beta(Wx)_2}}{e^{\beta(Wx)_1}+e^{\beta(Wx)_2}}
    =\tanh\!\left(\frac{\beta m}{2}\right),
\end{equation}
so fixed points are exactly the roots of
\begin{equation}
    m=\tanh\!\left(\frac{\beta m}{2}\right).
\end{equation}
Let $a=\beta/2$.  The equation is odd.  If $a\le1$, then for $m>0$,
\begin{equation}
    \tanh(am)<am\le m,
\end{equation}
so the only root is $m=0$.  If $a>1$, then $g(m)=\tanh(am)-m$ satisfies $g(0)=0$, $g'(0)=a-1>0$, and $g(1)=\tanh(a)-1<0$, so a positive root exists.  It is unique because $r(m)=\tanh(am)/m$ is strictly decreasing on $(0,1]$:
\begin{equation}
    r'(m)=\frac{am\sech^2(am)-\tanh(am)}{m^2}<0,
\end{equation}
since $y\sech^2y<\tanh y$ for every $y>0$, equivalently $y<\sinh y\cosh y$.  Oddness gives exactly two nonzero roots $\pm m_\beta$.  Finally, $\Tspace=\mathrm{span}\{(1,-1)\}$ and $W(1,-1)^\top=(1,-1)^\top$, hence $\lambda_{\max}(\Wt)=1$.
\end{proof}

\section{A worked logit feedback example: from high-noise stability to the responsive-but-monostable regime}
\label{app:responsive-monostable-example}

This appendix proves Proposition~\ref{prop:worked-two-action-phase-diagram} in full detail.  The goal is to make explicit what the constant-two improvement changes in the phase diagram of a concrete logit feedback system.

Consider the two-action simplex
\begin{equation}
    \Delta_2=\{x\in\mathbb R^2_{\ge0}:x_1+x_2=1\}
\end{equation}
and the affine feedback matrix
\begin{equation}
    W=\begin{pmatrix}0&-1\\-1&0\end{pmatrix}.
\end{equation}
The logit response map is
\begin{equation}
    F_\beta(x)=\softmax(\beta Wx),
\end{equation}
and the continuous-time logit adjustment dynamics are
\begin{equation}
    \dot x=F_\beta(x)-x.
\label{eq:worked-logit-ode-x}
\end{equation}

\paragraph{Step 1: the tangent feedback norm is one.}
The tangent space of $\Delta_2$ is
\begin{equation}
    \Tspace=\{v\in\mathbb R^2:v_1+v_2=0\}
    =\mathrm{span}\{(1,-1)\}.
\end{equation}
Since
\begin{equation}
    W(1,-1)^\top=(1,-1)^\top,
\end{equation}
the restriction of $\Pi W\Pi$ to $\Tspace$ acts as multiplication by $1$.  Therefore
\begin{equation}
    \|\Pi W\Pi\|_{\Tspace\to\Tspace}=1.
\end{equation}
The classical unit-softmax contraction certificate would require
\begin{equation}
    \beta\|\Pi W\Pi\|_{\Tspace\to\Tspace}<1,
\end{equation}
hence it certifies stability only for $\beta<1$.  By contrast, Theorem~\ref{thm:main} requires
\begin{equation}
    \frac{\beta}{2}\|\Pi W\Pi\|_{\Tspace\to\Tspace}<1,
\end{equation}
which is exactly
\begin{equation}
    \beta<2.
\end{equation}

\paragraph{Step 2: reduction to one scalar variable.}
Write the population imbalance as
\begin{equation}
    m=x_1-x_2\in[-1,1],
\end{equation}
so that
\begin{equation}
    x_1=\frac{1+m}{2},
    \qquad
    x_2=\frac{1-m}{2}.
\end{equation}
For this $x$,
\begin{equation}
    Wx=
    \left(
        -\frac{1-m}{2},
        -\frac{1+m}{2}
    \right).
\end{equation}
Hence the payoff difference is
\begin{equation}
    (Wx)_1-(Wx)_2=m.
\end{equation}
The softmax difference therefore satisfies
\begin{equation}
\begin{aligned}
    F_\beta(x)_1-F_\beta(x)_2
    &=
    \frac{\exp(\beta(Wx)_1)-\exp(\beta(Wx)_2)}
         {\exp(\beta(Wx)_1)+\exp(\beta(Wx)_2)}  \\
    &=
    \tanh\!\left(\frac{\beta[(Wx)_1-(Wx)_2]}{2}\right)  \\
    &=
    \tanh\!\left(\frac{\beta m}{2}\right).
\end{aligned}
\end{equation}
Taking the difference between the two coordinates in \eqref{eq:worked-logit-ode-x} gives the scalar dynamics
\begin{equation}
    \dot m
    =
    \tanh\!\left(\frac{\beta m}{2}\right)-m.
\label{eq:worked-m-dynamics}
\end{equation}
Thus fixed points are exactly the solutions of
\begin{equation}
    m=\tanh\!\left(\frac{\beta m}{2}\right).
\label{eq:worked-m-fixed}
\end{equation}

\paragraph{Step 3: uniqueness and convergence below the bifurcation threshold.}
Let
\begin{equation}
    a=\frac{\beta}{2},
    \qquad
    g(m)=\tanh(am)-m.
\end{equation}
If $\beta<2$, then $a<1$.  For $m>0$,
\begin{equation}
    \tanh(am)<am<m,
\end{equation}
so
\begin{equation}
    g(m)<0.
\end{equation}
By oddness, for $m<0$ we have $g(m)>0$.  Therefore the only zero of $g$ in $[-1,1]$ is $m=0$.

The same sign argument proves global convergence.  If $m(t)>0$, then $\dot m(t)<0$; if $m(t)<0$, then $\dot m(t)>0$.  Hence every trajectory moves monotonically toward $0$ and cannot cross it.  More quantitatively, for $m>0$,
\begin{equation}
    \dot m
    =
    \tanh(am)-m
    \le
    am-m
    =
    -(1-a)m.
\end{equation}
The symmetric inequality holds for $m<0$.  Therefore
\begin{equation}
    |m(t)|
    \le
    e^{-(1-a)t}|m(0)|
    =
    e^{-(1-\beta/2)t}|m(0)|.
\end{equation}
Thus for every $\beta<2$, the unique equilibrium $m=0$ is globally exponentially stable.

At the critical value $\beta=2$, the equation becomes
\begin{equation}
    \dot m=\tanh(m)-m.
\end{equation}
For $m>0$, $\tanh(m)<m$, and for $m<0$, $\tanh(m)>m$, so the equilibrium is still unique and globally attracting, although the linear exponential rate vanishes at the origin.  This is the bifurcation point.

\paragraph{Step 4: appearance of multiple stable equilibria above the threshold.}
Now suppose $\beta>2$, so $a>1$.  Then
\begin{equation}
    g(0)=0,
    \qquad
    g'(0)=a-1>0,
\end{equation}
while
\begin{equation}
    g(1)=\tanh(a)-1<0.
\end{equation}
By continuity, there exists at least one positive root $m_\beta\in(0,1)$.

This positive root is unique.  Indeed, for $m>0$ define
\begin{equation}
    r(m)=\frac{\tanh(am)}{m}.
\end{equation}
A positive fixed point is exactly a solution of $r(m)=1$.  We show that $r$ is strictly decreasing.  Its derivative is
\begin{equation}
    r'(m)
    =
    \frac{am\sech^2(am)-\tanh(am)}{m^2}.
\end{equation}
Let $y=am>0$.  The numerator has the same sign as
\begin{equation}
    y\sech^2(y)-\tanh(y).
\end{equation}
This quantity is negative because
\begin{equation}
    y\sech^2(y)<\tanh(y)
    \quad\Longleftrightarrow\quad
    y<\sinh(y)\cosh(y),
\end{equation}
and the last inequality holds for every $y>0$.  Hence $r'(m)<0$ on $(0,1]$, so there is exactly one positive root.  By oddness, there is exactly one negative root, denoted $-m_\beta$.

The stability picture follows from the sign of $g$.  Since $g'(0)>0$, the origin is unstable.  On $(0,m_\beta)$ one has $g(m)>0$, and on $(m_\beta,1]$ one has $g(m)<0$, so $m_\beta$ is attracting.  By symmetry, $-m_\beta$ is also attracting.  Therefore for every $\beta>2$, the system has three equilibria:
\begin{equation}
    -m_\beta,\qquad 0,\qquad m_\beta,
\end{equation}
with the two nonzero equilibria stable and the symmetric equilibrium unstable.

\paragraph{Conclusion.}
The old unit-softmax condition certifies only $\beta<1$, a strongly noisy regime in which feedback is heavily damped.  The constant-two condition certifies $\beta<2$, which is the entire pre-bifurcation region of this model.  The newly certified interval
\begin{equation}
    1\le\beta<2
\end{equation}
is therefore qualitatively important: agents are already responsive to payoff differences, but the system is still monostable and globally predictable.  The constant-two theorem turns this regime from an previously uncertified interval into a rigorous global-stability region.

\section{Consequences and comparisons}
\label{sec:consequences-comparison}

\subsection{A replacement principle for affine softmax-response contractions}
\label{subsec:replacement-principle}

\begin{corollary}[Constant-two replacement for softmax-response contraction templates]
\label{cor:response-template}
Let $G:\DeltaProd\to\mathbb R^N$ be an affine payoff map $G(x)=Wx+b$, and let
\begin{equation}
    F(x)=\sigma_{\rm blk}(B_\beta G(x))
\end{equation}
be the corresponding block-softmax response map.  Any convergence statement whose sufficient condition is that $F$ is a contraction on the product simplex remains valid under
\begin{equation}
    \frac12\|B_\beta\Pi W\Pi\|_{\Tspace\to\Tspace}<1.
\end{equation}
In particular, in affine logit systems the usual unit-softmax sufficient condition
\begin{equation}
    \|B_\beta W\|_2<1
\end{equation}
obtained by treating softmax as a unit-scale Lipschitz map can be replaced by the tangent-projected constant-two condition
\begin{equation}
    \|B_\beta\Pi W\Pi\|_{\Tspace\to\Tspace}<2.
\end{equation}
For a single simplex this reduces to $\beta\|\Pi W\Pi\|_{\Tspace\to\Tspace}<2$.
\end{corollary}

Many softmax-response analyses first prove that a response map is contractive and then invoke Banach's theorem or a standard contraction-flow argument.  In the affine endogenous-payoff case, Theorem~\ref{thm:product-contraction} supplies the sharper contraction input.  The result does not automatically cover non-Euclidean, coordinate-wise, asynchronous, or stochastic-approximation analyses unless their proof reduces to this deterministic Euclidean contraction property.

\subsection{Separation from block \texorpdfstring{$\ell_1$}{l1}/Dobrushin certificates}
\label{subsec:dobrushin-separation}

Dobrushin-style criteria control absolute block influence, usually in total variation or $\ell_1$ geometry.  The Euclidean tangent certificate instead sees signed cancellation.

\begin{proposition}[A product-simplex separation from Dobrushin-type certificates]
\label{prop:l2-dobrushin-separation}
For every power of two $m$ and every $\alpha\in(2/\sqrt m,2)$, there exists a binary product-simplex logit system with $m$ blocks and unit inverse temperatures such that
\begin{equation}
    \frac12\|\Pi W\Pi\|_{\Tspace\to\Tspace}=\frac\alpha2<1,
\end{equation}
so Theorem~\ref{thm:product-contraction} certifies global uniqueness and convergence, while the natural block $\ell_1$/Dobrushin certificate based on
\begin{equation}
    C_{ab}=\frac12\|\Pi_a W^{ab}\Pi_b\|_{1\to1}
\end{equation}
fails because
\begin{equation}
    \rho(C)=\frac{\alpha\sqrt m}{2}>1.
\end{equation}
\end{proposition}

Taking $\alpha=1$ and $m\ge8$ gives a system whose Euclidean contraction factor is $1/2$, while the block $\ell_1$ influence radius already exceeds one.  This signed block-cancellation example, visualized in Figure~\ref{fig:dobrushin-separation}, shows that the product-simplex spectral condition is not just a different notation for a Dobrushin bound.

\subsection{What the factor two buys}
\label{subsec:what-factor-two-buys}

The improvement from $1$ to $2$ is most visible in the interval
\begin{equation}
    1\le \beta\|\Wt\|_2<2.
\end{equation}
Classical unit-scale contraction no longer certifies this region, yet the two-action potential model remains unique throughout it.  The results here show that this interval is part of the globally stable high-temperature regime once the exact softmax geometry is used.  In applications where $\beta$ is an inverse noise level or rationality parameter, this doubles the certified uniqueness region relative to the coarse certificate.  In dynamical terms, the same enlargement applies to the region where repeated logit response and continuous logit adjustment are guaranteed to converge from every initialization.

\section{Proofs for consequences and comparisons}
\label{app:consequences-comparison}

\begin{proof}[Proof of Corollary~\ref{cor:response-template}]
Theorem~\ref{thm:product-contraction} states exactly that
\begin{equation}
    \frac12\|B_\beta\Pi W\Pi\|_{\Tspace\to\Tspace}<1
\end{equation}
implies that the block-softmax response map $F$ is a global contraction on the product simplex.  Therefore any result whose proof requires only this contraction property can be invoked with this sharper sufficient condition.  The displayed unit-softmax condition is the one obtained from the coarse estimate
\begin{equation}
    \|\sigma_{\rm blk}(B_\beta z)-\sigma_{\rm blk}(B_\beta z')\|_2
    \le \|B_\beta(z-z')\|_2
\end{equation}
combined with $z-z'=W(x-y)$.  Theorem~\ref{thm:product-contraction} replaces this ambient bound by the covariance bound and the tangent projection, giving the condition
\begin{equation}
    \|B_\beta\Pi W\Pi\|_{\Tspace\to\Tspace}<2.
\end{equation}
The single-simplex statement is the special case $m=1$ and $B_\beta=\beta I$.
\end{proof}

\begin{proof}[Proof of Proposition~\ref{prop:l2-dobrushin-separation}]
Let every block be binary, so $\Delta_a=\Delta_2$ and $\Tspace_a=\mathrm{span}\{u\}$ with
\begin{equation}
    u=\frac1{\sqrt2}(1,-1)^\top.
\end{equation}
For every power of two $m$, take a Sylvester Hadamard matrix $H_m\in\{\pm1\}^{m\times m}$, so that $m^{-1/2}H_m$ is orthogonal.  Define the block interaction matrix by
\begin{equation}
    W^{ab}=\frac{\alpha}{\sqrt m}(H_m)_{ab}uu^\top,
    \qquad a,b\in[m],
\end{equation}
and take all inverse temperatures equal to one.  Since $u\in\Tspace_a$ and $uu^\top$ annihilates the all-ones direction, $\Pi_a W^{ab}\Pi_b=W^{ab}$.

Represent a tangent vector $h\in\Tspace$ as $h^b=t_b u$.  Then
\begin{equation}
    (\Pi W\Pi h)^a
    =\sum_{b=1}^m \frac{\alpha}{\sqrt m}(H_m)_{ab}uu^\top(t_bu)
    =\left(\sum_{b=1}^m\frac{\alpha}{\sqrt m}(H_m)_{ab}t_b\right)u.
\end{equation}
Thus the tangent operator is represented in the orthonormal block basis by
\begin{equation}
    A=\frac{\alpha}{\sqrt m}H_m,
\end{equation}
so $\|\Pi W\Pi\|_{\Tspace\to\Tspace}=\|A\|_2=\alpha$.  Hence
\begin{equation}
    \frac12\|\Pi W\Pi\|_{\Tspace\to\Tspace}=\frac\alpha2<1.
\end{equation}
The product-simplex contraction theorem therefore certifies global uniqueness and convergence.

For the block $\ell_1$ influence matrix, note that
\begin{equation}
    uu^\top=\frac12\begin{pmatrix}1&-1\\-1&1\end{pmatrix},
    \qquad
    \|uu^\top\|_{1\to1}=1.
\end{equation}
Therefore
\begin{equation}
    C_{ab}=\frac12\|W^{ab}\|_{1\to1}
    =\frac{\alpha}{2\sqrt m}
\end{equation}
for all $a,b$.  Hence $C$ is the all-ones matrix multiplied by $\alpha/(2\sqrt m)$, and
\begin{equation}
    \rho(C)=m\cdot\frac{\alpha}{2\sqrt m}=\frac{\alpha\sqrt m}{2}>1
\end{equation}
by the assumption $\alpha>2/\sqrt m$.  The natural block $\ell_1$/Dobrushin certificate therefore fails.
\end{proof}

\section{Numerical illustrations}
\label{app:numerical-illustrations}

This appendix collects small deterministic numerical illustrations.  Their purpose is to show how the certificates behave in finite-dimensional instances beyond the two-action phase diagram.  All figures in this section are generated by \texttt{scripts/make\_figures.py} with fixed random seeds and no external data.

\subsection{Two-action phase diagram}
\label{app:num-two-action-phase-diagram}

The two-action model is exactly reducible to the scalar equation analyzed in Proposition~\ref{prop:worked-two-action-phase-diagram}.  The following figure visualizes that closed-form phase diagram and the associated deterministic logit-adjustment trajectories.

\begin{figure}[ht]
\centering
\includegraphics[width=\textwidth]{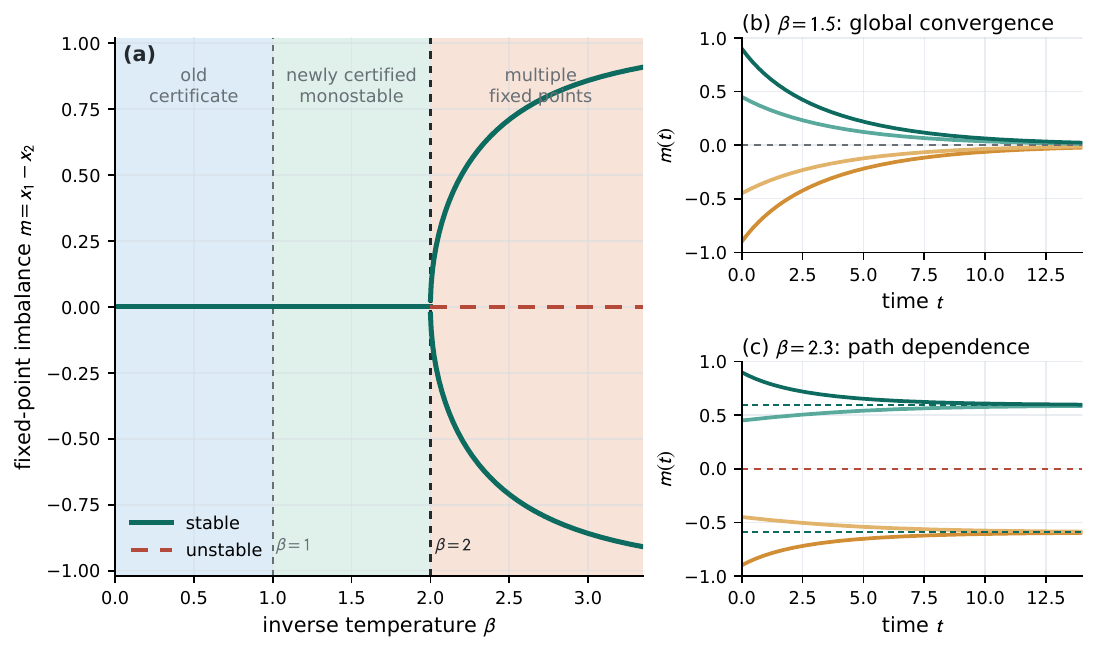}
\caption{
\textbf{Numerical illustration of the two-action logit phase diagram.}
Panel (a) plots the fixed points of \(m=\tanh(\beta m/2)\).  Solid curves are stable equilibria and the dashed curve is unstable.  The old unit-softmax certificate covers only \(\beta<1\), while the covariance-calibrated theorem certifies the full pre-bifurcation interval \(\beta<2\).  Panels (b) and (c) show deterministic logit-adjustment trajectories from several initial imbalances, illustrating global convergence below the threshold and path dependence above it.
}
\label{fig:logit-phase-diagram}
\end{figure}

Figure~\ref{fig:logit-phase-diagram} visualizes the exact scalar phase diagram behind Proposition~\ref{prop:worked-two-action-phase-diagram}.  The figure is a deterministic illustration of the closed-form fixed-point equation and ODE reduction.

\subsection{Certificate gain for random affine logit systems}
\label{app:num-certificate-gain}

For a single-simplex affine logit system, the coarse ambient certificate gives
\begin{equation}
    \beta_{\rm old}=\frac{1}{\|W\|_2},
    \qquad
    \beta_{\rm new}=\frac{2}{\|\Pi W\Pi\|_{\Tspace\to\Tspace}}.
\end{equation}
We report the certified range gain
\begin{equation}
    \frac{\beta_{\rm new}}{\beta_{\rm old}}
    =
    \frac{2\|W\|_2}{\|\Pi W\Pi\|_{\Tspace\to\Tspace}},
\label{eq:num-range-gain}
\end{equation}
or equivalently its reciprocal \(\beta_{\rm old}/\beta_{\rm new}\).  For ordinary Gaussian matrices, the tangent norm is usually close to the ambient norm, so the visible gain is primarily the factor two.  To isolate the geometric correction, we also generate payoff-shift components of the form
\begin{equation}
    W=W_{\rm tan}+\mathbf 1 a^\top+b\mathbf 1^\top,
\end{equation}
with the symmetric variant \(W=W_{\rm tan}+\mathbf 1 a^\top+a\mathbf 1^\top\).  Since
\begin{equation}
    \Pi(\mathbf 1 a^\top+b\mathbf 1^\top)\Pi=0,
\end{equation}
these components can inflate \(\|W\|_2\) without changing the tangent interaction seen by softmax.

\begin{figure}[ht]
\centering
\includegraphics[width=0.88\textwidth]{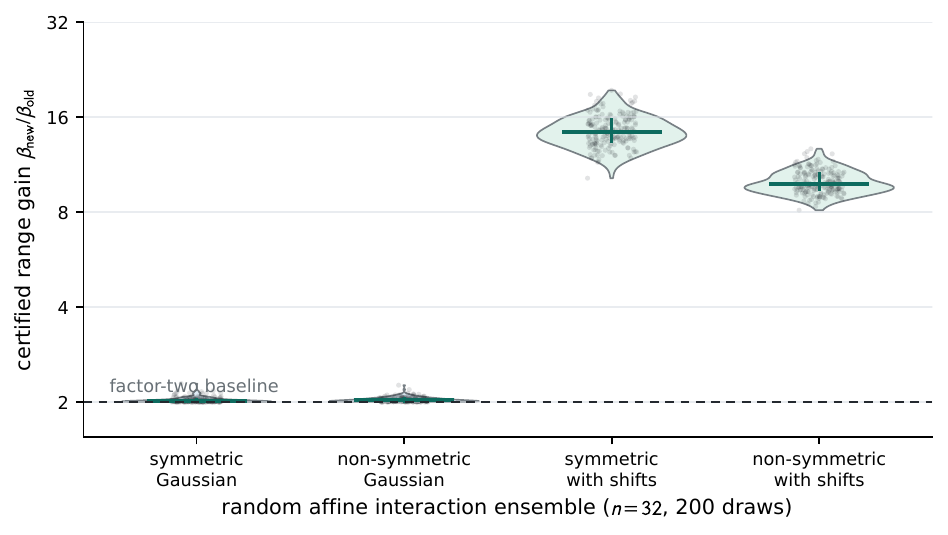}
\caption{
\textbf{Certificate gain for random affine logit systems.}
Each violin summarizes 200 fixed-seed draws with \(n=32\).  The plotted quantity is \(\beta_{\rm new}/\beta_{\rm old}=2\|W\|_2/\|\Pi W\Pi\|_{\Tspace\to\Tspace}\).  Ordinary Gaussian symmetric and non-symmetric matrices mainly show the factor-two improvement.  Adding payoff-shift components leaves the tangent certificate unchanged but shrinks the ambient certificate, so the certified inverse-temperature range expands by more than a constant factor.
}
\label{fig:random-certificate-gain}
\end{figure}

\subsection{Convergence in the newly certified regime}
\label{app:num-newly-certified-convergence}

The next illustration checks nontrivial instances in which the old certificate fails but the new certificate succeeds.  We generate \(12\) non-symmetric shifted Gaussian systems with \(n=20\), random biases, and choose
\begin{equation}
    \beta=0.72\,\beta_{\rm new}.
\end{equation}
Because \(\beta_{\rm new}/\beta_{\rm old}\ge2\), this choice satisfies
\begin{equation}
    \beta_{\rm old}<\beta<\beta_{\rm new}
\end{equation}
for every displayed instance.  Thus the ambient unit-softmax certificate fails, while Theorem~\ref{thm:main} certifies contraction.  For each instance, we run Picard iteration from ten independent Dirichlet initializations and plot the maximum pairwise distance among the iterates.

\begin{figure}[ht]
\centering
\includegraphics[width=\textwidth]{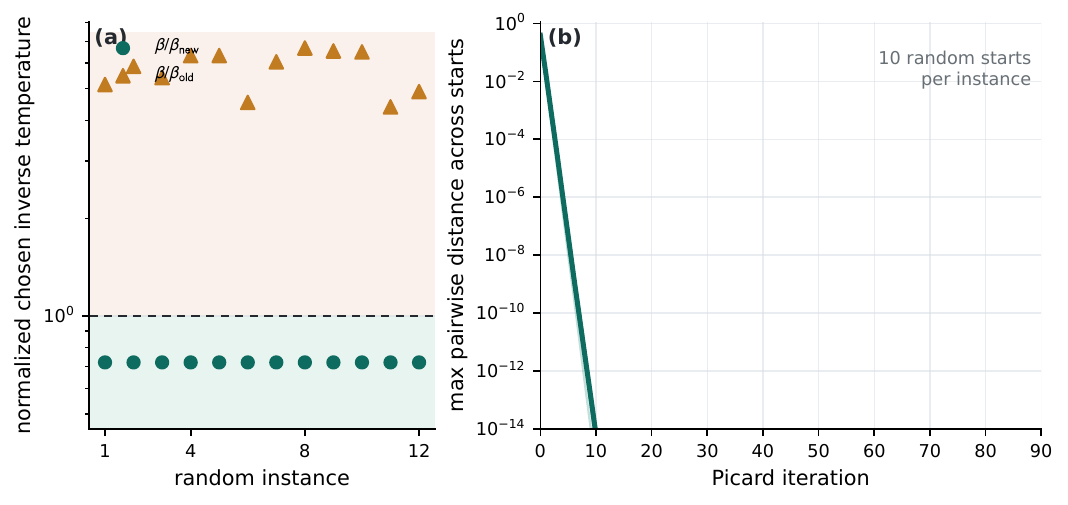}
\caption{
\textbf{Picard convergence in the newly certified interval.}
Panel (a) shows that the chosen \(\beta\) lies above the old threshold and below the new threshold for each random shifted non-symmetric instance.  Panel (b) shows the collapse of ten random initializations per instance under Picard iteration.  The plot visualizes deterministic Picard convergence in finite-dimensional random affine systems certified by the theorem.
}
\label{fig:newly-certified-convergence}
\end{figure}

\subsection{Product-simplex separation from Dobrushin-style certificates}
\label{app:num-dobrushin-separation}

The product-simplex theorem exploits signed Euclidean tangent cancellation.  The Hadamard block construction in Proposition~\ref{prop:l2-dobrushin-separation} gives, for \(\alpha=1\),
\begin{equation}
    q_{\ell_2}
    =
    \frac12\|\Pi W\Pi\|_{\Tspace\to\Tspace}
    =
    \frac12,
    \qquad
    \rho(C)=\frac{\sqrt m}{2}.
\end{equation}
Thus the Euclidean product-simplex certificate remains valid as the number of binary blocks grows, while the natural block \(\ell_1\)/Dobrushin-style certificate fails once \(m\ge8\).

\begin{figure}[ht]
\centering
\includegraphics[width=0.74\textwidth]{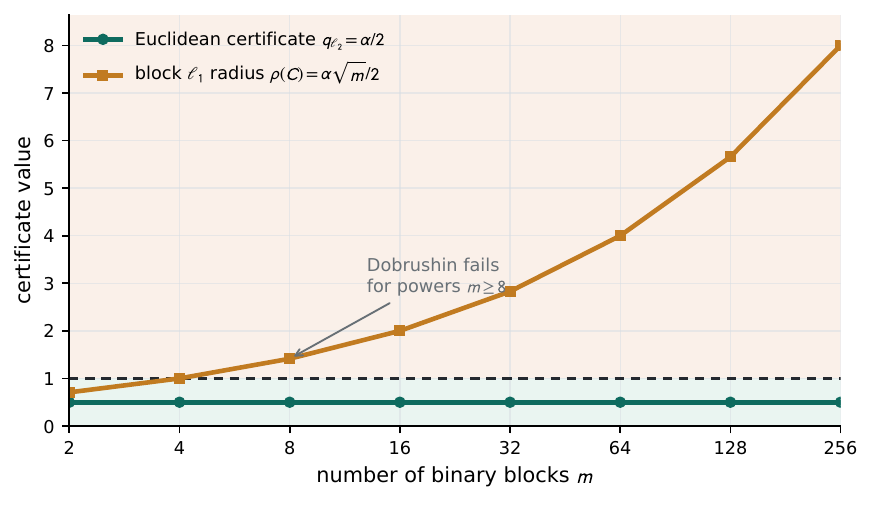}
\caption{
\textbf{Signed block cancellation separates the Euclidean and Dobrushin certificates.}
For the Hadamard block construction with \(\alpha=1\), the Euclidean contraction factor remains \(q_{\ell_2}=1/2\) for all powers-of-two block counts, while the natural block \(\ell_1\) influence radius grows as \(\rho(C)=\sqrt m/2\) and exceeds the uniqueness threshold once \(m\ge8\).
}
\label{fig:dobrushin-separation}
\end{figure}

\section{Sharpness and comparison}
\label{sec:discussion}

\paragraph{Sharpness taxonomy.}
The word sharp has several meanings here, and separating them avoids overclaiming.  First, the softmax covariance constant is sharp: $\sup_p\|\Sigma(p)\|_2=1/2$, with equality on a two-point distribution.  Second, the arbitrary-interaction theorem is sharp as a universal Euclidean tangent spectral contraction certificate: no proof that depends only on the global covariance norm and $\|\Wt\|_2$ can improve the constant $2$.  Third, in the symmetric potential class, the constant is sharp as an entropy-versus-interaction curvature threshold.  Fourth, the two-action symmetric model realizes this threshold as an actual uniqueness boundary: below $2$ there is one fixed point, and above $2$ there are three. Special matrices may remain unique far beyond the global contraction region.

\paragraph{Relation to Dobrushin-style $\ell_1$ criteria.}
Dobrushin criteria control how much one block's conditional distribution can change when another block is perturbed, usually in total variation or $\ell_1$ geometry.  The present result is complementary.  Since the softmax Jacobian also satisfies an $\ell_1$ sensitivity bound of scale $1/2$, one can form a block influence matrix
\[
    C_{ab}=\frac{\beta_a}{2}\|\Pi_a W^{ab}\Pi_b\|_{1\to1}
\]
and obtain a Dobrushin-style certificate $\rho(C)<1$.  This may be sharper for sparse, directed, or nearly triangular influence patterns.  For example, a strictly upper-triangular block influence matrix has Dobrushin spectral radius zero even when individual block influences are large, while the Euclidean operator norm of the corresponding signed tangent interaction can be large.  The Euclidean spectral condition in Theorem~\ref{thm:product-contraction} is instead designed to exploit cancellations, low-rank structure, and symmetric curvature.  Proposition~\ref{prop:l2-dobrushin-separation} makes the opposite separation explicit: signed orthogonal block couplings can have bounded tangent spectral norm while their absolute block influence matrix has spectral radius of order $\sqrt m$ larger.  Neither condition uniformly dominates the other; they are different high-temperature projections of the same softmax covariance geometry.

\paragraph{What the factor two buys.}
The improvement from $1$ to $2$ is most visible in the high-temperature interval
\[
    1\le \beta\|\Wt\|_2<2.
\]
Classical unit-scale contraction no longer certifies this regime, yet the two-action potential model remains unique throughout it.  The results here show that this interval is part of the globally stable high-temperature regime once the exact softmax geometry is used.  In applications where $\beta$ is an inverse noise level or rationality parameter, this means that the guaranteed uniqueness region is twice as large as the one obtained from the coarse certificate.  In dynamical terms, the same enlargement applies to the region where repeated logit response and continuous logit adjustment are guaranteed to converge from every initialization.

\end{document}